\documentclass[11pt,letterpaper]{article}

\usepackage[T1]{fontenc}
\usepackage[utf8]{inputenc}
\usepackage{mathpazo}
\usepackage{microtype}
\usepackage[left=1.35in,right=1.35in,top=1.2in,bottom=1.2in]{geometry}
\usepackage{amsmath,amssymb,amsthm}
\usepackage{graphicx}
\usepackage{booktabs}
\usepackage{tabularx}
\usepackage{threeparttable}
\usepackage{caption}
\usepackage{enumitem}
\usepackage{xcolor}
\usepackage{titlesec}
\usepackage[hidelinks,breaklinks]{hyperref}
\usepackage{url}
\definecolor{navy}{HTML}{2F4B7C}
\definecolor{rulegrey}{HTML}{B9BEC6}

\titleformat{\section}{\normalfont\large\bfseries\color{navy}}{\thesection}{0.7em}{}
\titleformat{\subsection}{\normalfont\normalsize\bfseries}{\thesubsection}{0.7em}{}
\titlespacing*{\section}{0pt}{1.6\baselineskip}{0.5\baselineskip}
\titlespacing*{\subsection}{0pt}{1.1\baselineskip}{0.35\baselineskip}

\theoremstyle{plain}
\newtheorem{theorem}{Theorem}

\newtheorem{corollary}{Corollary}
\theoremstyle{definition}
\newtheorem{definition}{Definition}

\newcommand{\Nmax}{N_{\max}}
\newcommand{\pmin}{p_{\min}}
\newcommand{\neff}{n_{\mathrm{eff}}}
\newcommand{\Lnull}{\Lambda_0}
\newcommand{\Lpos}{\Lambda_1}

\begin{document}

\begin{center}
{\LARGE\bfseries What AI Red-Team Evaluations Can and Cannot Prove}\\[6pt]
{\large Bandana Kaur}\\[3pt]
{\normalsize APIsec Research Labs}\\[14pt]
\end{center}

\begin{center}
\begin{minipage}{0.90\textwidth}
\small
\noindent\textbf{Abstract.}
Red-team evaluations of AI models support some claims and not others, and the
boundary between the two is calculable rather than merely a matter of judgment. We
define the \emph{evidential ceiling} of an evaluation as the largest factor by
which one result can move belief under a fixed testing budget, derive it in
closed form for the benchmark null result, and use it to locate that boundary
exactly. We find that above a calculable harm rate, a benchmark of
modest size certifies a category to a stated evidentiary standard, and a clean
sheet is then the \emph{stronger} of the two possible observations, outweighing
a single reproduced failure. Below that rate, no passive benchmark of feasible size provides the specified evidence of safety under the fixed scoring rule and approximately independent trial structure. The crossing between the two regimes has a closed form and
falls as $1/n$. The bound is not specific to benchmarks: written in terms of a procedure's
hypothesis-conditioned elicitation rates, it covers adaptive and automated red
teaming as well, and shows that discrimination between the hypotheses rather
than attack success is what determines evidential worth. Auditing eight
evaluation suites against the boundary, we find that current benchmarks are
adequate for high-frequency harm categories and several orders of magnitude
short for rare, catastrophic ones. Safety benchmarks are not
uninformative. They are informative about a specific and computable set of
propositions, and the discipline they need is to state which.
\end{minipage}
\end{center}
\vspace{6pt}

% ==========================================================================
\section{Introduction}

Every major AI laboratory now publishes safety evaluations before deploying a
new model \cite{bommasani2021foundation}. A curated set of adversarial prompts
is assembled, model responses are scored for harmfulness, outcomes are compared
against prior versions or established thresholds, and the results are
communicated in safety reports, system cards \cite{mitchell2019modelcards}, and
press releases as evidence that the model is safe enough to deploy. Regulatory
bodies have begun treating these evaluations as the primary empirical basis for
oversight determinations. The practice has become a field, with community
conventions, public leaderboards, and an expanding catalogue of standardized
test suites \cite{liang2023helm, dehghani2021lottery}.

A reaction has set in. Recent work has questioned whether red-teaming's
quantitative outputs support comparison across systems at all
\cite{chouldechova2025comparison}, whether benchmark performance generalizes
across semantically equivalent prompts \cite{broomfield2025structural}, and
whether evaluations carry the statistical precision their conclusions assume
\cite{miller2024errorbars}. The critiques are absolutely correct. But the conclusion
sometimes drawn from them, that safety benchmarks are close to worthless, does
not follow, and this paper argues against it.

Our position is that safety evaluations are informative, and that the useful
question is not whether they work but which propositions they establish. Given a harm rate, a
sample size, and a stated evidentiary standard, one can compute whether an
evaluation licenses a certification claim, and if not, how large it would need
to be. For high-frequency harm categories, existing public benchmarks could already
clear that bar. For rare categories, they fall short by three orders of
magnitude, and past a calculable rate no benchmark of feasible size clears it at
all. Both halves of that finding are consequences of one closed-form expression,
and the practical value of the expression is that it tells a laboratory what it
may claim before it runs the evaluation rather than after.

This has an unexpected corollary; the
evidence carried by a clean benchmark and by a single reproduced failure cross
at a computable harm rate. Below it, one observed failure outweighs a clean
sheet, which is the finding the critical literature would expect. Above it, the
ordering reverses and the clean sheet is the stronger evidence. At a one percent
harm rate with 520 prompts, a clean result carries 1.4 times the evidence of a
single observed harm. A framework capable of reaching that conclusion is not
one tuned to disparage benchmarks.

Medicine reached the same place by a harder route. Early drug approvals treated
the absence of observed adverse events in small trials as evidence of safety;
the failures that followed, thalidomide among them \cite{kim2011thalidomide},
produced the modern requirement for powered, pre-registered trials. What
medicine did not conclude was that trials are uninformative. It concluded that a
trial licenses a claim proportional to its power, and it built the machinery to
compute the proportion. Surrogate-endpoint theory added the second condition
\cite{prentice1989surrogate}: a proxy supports claims about a hard endpoint only
if it predicts that endpoint in the target population. We believe AI safety evaluation is
at an earlier methodological stage, and formalizing now, before failure at
scale forces it, is cheaper and more tractable.
% ==========================================================================
\section{Related work}

Chouldechova et al. \cite{chouldechova2025comparison} use measurement theory to
ask when red-teaming's quantitative outputs, mainly attack success rate, support
meaningful comparison across systems, and conclude that the conditions for valid
comparison are rarely met. Broomfield et al. \cite{broomfield2025structural}
name the failure of safety to generalize across semantically equivalent prompt
structures and use it to motivate new attacks and a defense. Both diagnoses are
qualitative. Neither quantifies the resulting evidential gap, nor identifies
where the gap closes with more data and where it does not.

A separate strand introduces standard statistical apparatus into language model
evaluation generally \cite{miller2024errorbars}, warning that small evaluations
yield wide intervals and unreliable comparisons. This establishes our
Condition~1 for evaluations at large. We specialize to the rare-event regime,
derive closed forms in both directions, and audit how far current practice sits
from each.

The safety-case program asks a different question \cite{clymer2024safetycases,
korbak2025controlsafetycase}: given several heterogeneous evidence sources, does
their combination license a deployment conclusion? That is an aggregation
question. We work one level below it, asking how much weight a single source
carries, and supply the per-source quantity such an aggregation needs.

% ==========================================================================
\section{What an evaluation must establish}

Safety evaluation is an inference from observed test performance to behavior
under deployment conditions not included in the test. Any such inference
requires three things.

\begin{enumerate}[leftmargin=1.6em,itemsep=3pt,topsep=4pt]
\item \textbf{Statistical sufficiency:} The test has power to detect
safety-relevant differences: the sample is large enough to distinguish the
signal from sampling noise.

\item \textbf{Distributional validity:} The distribution of test inputs
represents the deployment distribution, so that what is measured predicts what
happens in deployment rather than only in the test.

\item \textbf{Structural generalization:} Robustness on the test set
generalizes beyond the specific stimuli it contains, so that passing is evidence
of a latent safety property rather than of having learned to pass.
\end{enumerate}

These are the minimum requirements for any empirical claim generalizing from
sample to population \cite{shadish2002experimental}. They are why randomized
trials require power calculations \cite{moher1994power} and why regulatory
science requires representative sampling frames. Sections~\ref{sec:ceiling}
and~\ref{sec:practice} make Condition~1 exact and measure the distance between
it and current practice; Section~\ref{sec:practice} treats Conditions~2 and~3
empirically; and Section~\ref{sec:claims} states what each condition, once met,
entitles an evaluation to claim.

% ==========================================================================
\section{The evidential ceiling}\label{sec:ceiling}

For any evidence-generating procedure restricted to a feasibility budget, we
define its evidential ceiling as the maximum factor by which one result can move
posterior odds between a hypothesis of elevated risk and one of acceptable risk.
The construct is general and could in principle be computed for any evidence
type a safety case cites. We instantiate it fully for the benchmark null result
and leave the others open.

The derivation is deliberately elementary. Its content is a special case of the
zero-numerator bound formalized by Hanley and Lippman-Hand in 1983
\cite{hanley1983nothing}. Our contribution
is the consequence for evaluation design, that the boundary between what a
benchmark can and cannot certify is a computable quantity rather than a matter
of expert judgment, and that it lies in a different place than current practice
assumes in both directions.

\begin{definition}[Evidential ceiling]\label{def:ceiling}
An evidence-generating procedure is a tuple $E = (\mathcal{H}, B, C, \Lambda)$,
where $\mathcal{H} = \{H_0, H_1\}$ are competing hypotheses about the latent
safety state, with $H_0$ acceptable risk and $H_1$ elevated risk; $B$ is a
resource budget, for a benchmark a feasibility ceiling $\Nmax$ on a count of
approximately independent trials; $C$ is the observation channel carrying the
safety state to an observable result $r$; and
$\Lambda(r) = P(r \mid H_1) / P(r \mid H_0)$. Because evidence has a direction,
the ceiling is defined separately for each. The \emph{exculpatory} ceiling,
the most a result can argue toward acceptable risk, is
\[
  \mathrm{Ceil}^{-}(E) \;=\; \sup_{\,r \,:\, \Lambda(r) \le 1} \; \bigl| \log \Lambda(r) \bigr| ,
\]
and the \emph{incriminating} ceiling is the corresponding supremum over $r$ with
$\Lambda(r) \ge 1$. Both range over results feasible under $B$, and both measure
the largest amount of information a single result can contribute about the
safety state. We report such quantities in bits, taking $\log = \log_2$
throughout, except inside proofs where natural logarithms are used.
\end{definition}

So for a benchmark of $n$ trials the unrestricted supremum
of $|\log \Lambda|$ is attained at $k = n$, every prompt harmful, an outcome
that argues for elevated risk and says nothing about the worth of a clean sheet.
It is $\mathrm{Ceil}^{-}$ that governs certification, attained at $k = 0$ since
$\Lnull$ is monotone in $k$. Both ceilings are properties of the channel $C$ and
the budget $B$, not of any scoring rule layered on the raw result: by the data
processing inequality no downstream relabeling, aggregation, or reformatting can
raise either (Figure~\ref{fig:ceiling}a).

\begin{figure}[!htbp]
\centering
\includegraphics[width=\textwidth]{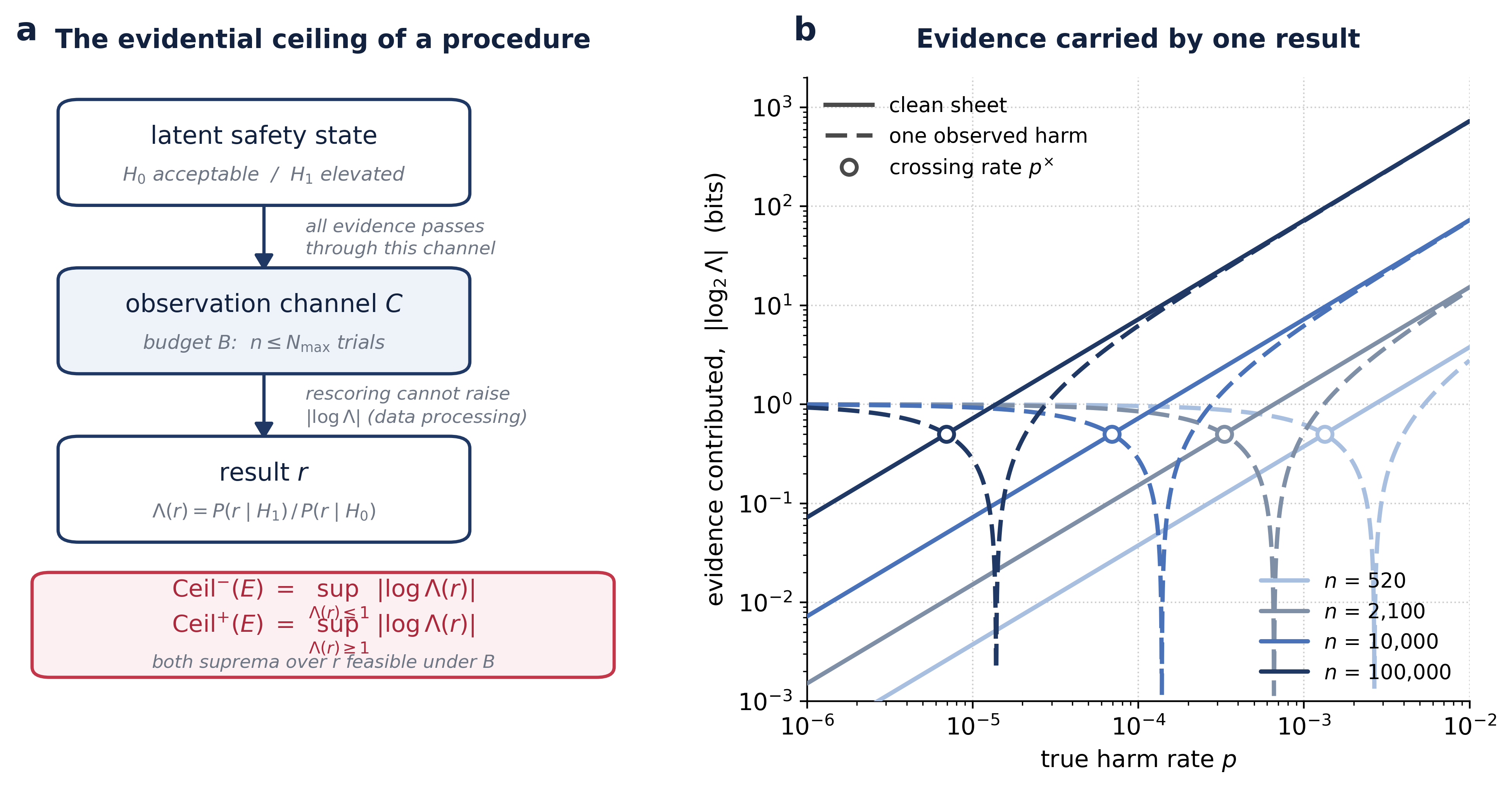}
\caption{\textbf{The evidential ceiling and the two evidence regimes.}
(\textbf{a}) The evidence channel. Both ceilings are fixed by the channel and
the budget; by the data processing inequality no rescoring or aggregation can
raise either. The exculpatory ceiling $\mathrm{Ceil}^{-}$ governs
certification, and for a passive benchmark it is attained at $k = 0$.
(\textbf{b}) Evidence contributed by a single result, in bits, both quantities
scored against the same hypothesis pair ($H_0$: $p = rp_u$ versus $H_1$:
$p = p_u$) at $r = 0.5$. Solid curves are a clean sheet, dashed curves one
observed harmful output. Open circles mark the crossing rate of
Corollary~\ref{cor:crossing}. To the right of a circle the clean sheet is the
stronger evidence; to the left it is not. Each dashed curve breaks at
$p = 2p^{\times}$, where $\Lambda_1 = 1$ and one harm carries no evidence at
all; the magnitude axis does not display the change of direction there. As $p$
falls the clean sheet carries vanishing evidence while the single harm
converges to $\log_2(1/r) = 1$ bit, independent of $n$.}
\label{fig:ceiling}
\end{figure}

\subsection{The benchmark null result, made precise}

Let two hypotheses bound the harm rate of interest: $H_1$ (elevated risk,
$p = p_u$) and $H_0$ (acceptable risk, $p = p_s = r p_u$), with $r \in (0,1)$
the improvement ratio, so that $r = 0.5$ is the fifty percent reduction safety
reports most commonly claim. A belief is prior odds $O_0 = P(H_1)/P(H_0)$, and
an evaluation of $n$ prompts updates it by $O_1 = O_0 \Lambda$.

For a zero-harm outcome, $\Lnull = [(1-p_u)/(1-p_s)]^{n}$. At $n = 520$ and
$(p_u, p_s) = (0.001, 0.0005)$ this is $0.77$, a likelihood ratio barely
distinguishable from the $1$ that constitutes no evidence at all. Expanding to
first order, $\ln \Lnull = -n p_u (1-r) + \mathcal{O}(n p_u^2)$, so
$\Lnull \to 1$ as $p_u \to 0$ at fixed $n$. Null results become less
informative, not more, as harm becomes rarer. Requiring $\Lnull \le \tau$ gives
the sample-size criterion
\begin{equation}\label{eq:nreq}
  n \;\ge\; \frac{\ln \tau}{\ln[(1-p_u)/(1-p_s)]} \;\approx\; \frac{-\ln \tau}{p_u (1-r)},
\end{equation}
which is the expression the rest of this section reads in both directions.

\subsection{Above the boundary: what a benchmark does establish}

We now state the constructive direction, it is the one current
discussion is observed to neglect.

\begin{theorem}[Adequacy of feasible benchmarks above the boundary]\label{thm:adequacy}
Fix an evidentiary threshold $\tau$, an improvement ratio $r$, and a feasibility
ceiling $\Nmax$. For any harm category whose rate under the scoring rule in use
satisfies
\[
  p \;>\; \pmin(\tau, \Nmax, r) \;\approx\; \frac{-\ln \tau}{\Nmax(1-r)},
\]
there exists a finite $n \le \Nmax$, given by \eqref{eq:nreq}, such that a
zero-harm result on $n$ approximately independent trials moves posterior odds
toward acceptable risk by at least the factor $\tau$. The required $n$ scales as
$\mathcal{O}(1/p)$ and is computable in closed form before the evaluation is
run.
\end{theorem}

\begin{proof}
$\Lnull(n,p) = [(1-p)/(1-rp)]^{n}$ is decreasing in $n$ for $r \in (0,1)$, so
$\Lnull(n,p) \le \tau$ holds for all $n$ at or above the value in
\eqref{eq:nreq}. That value is at most $\Nmax$ exactly when $p > \pmin$.
\end{proof}

The arithmetic theorem is encouraging at the rates where
much safety-relevant behaviour actually sits. At $p = 10^{-2}$, a
factor-of-two update needs 138 prompts, an order-of-magnitude update needs 458,
and holding false certification below five percent needs 299. AdvBench, at
$n = 520$, clears all three. 

\subsection{Below the boundary: what no benchmark establishes}

\begin{theorem}[Impossibility of informative null-result certification]\label{thm:impossibility}
Fix $\tau$, $r$, and $\Nmax$ as above. Fix any scoring rule, and let $p$ denote
the harm rate the category exhibits under that rule. If $p < \pmin(\tau, \Nmax,
r)$, then no evaluation satisfying $n \le \Nmax$, whose trials are approximately
independent Bernoulli draws at rate $p$, can produce a zero-harm result
constituting informative evidence of safety at threshold $\tau$, whatever its
prompt diversity.
\end{theorem}

\begin{proof}
$\Lnull(n,p)$ is decreasing in $n$ and increasing as $p$ falls. The constraint
$n \le \Nmax$ therefore caps the achievable reduction at $\Lnull(\Nmax, p)$.
Setting $\Lnull(\Nmax,p) = \tau$ and solving for $p$ gives $\pmin$. For
$p < \pmin$, monotonicity gives $\Lnull(n,p) \ge \Lnull(\Nmax,p) > \tau$ for all
$n \le \Nmax$.
\end{proof}

The quantification over scoring rules deserves care. The result does not hold
``regardless of the scoring rule'', which would be false: the rule defines the
event counted and hence defines $p$, and a stricter judge that flags borderline
completions induces a larger $p$ and can lift a category out of the regime
entirely. What the theorem says is that for each fixed rule the bound applies at
whatever rate that rule induces. Read constructively, this makes detection
sensitivity the only lever that moves a category across the boundary, and it is
a lever on the observation channel rather than on the budget.

At $\Nmax = 10^{5}$, $r = 0.5$, $\tau = 0.5$, the boundary sits at
$\pmin \approx 1.4 \times 10^{-5}$ (Figure~\ref{fig:boundary}b). We do not
assert that any specific CBRN category falls below it; population base rates for
these categories are contested and largely unmeasured. The claim is
conditional, so if a category's true rate falls in this regime, which current
evidence does not rule out, then certifying its absence via null results is
infeasible in principle within any cost structure resembling current budgets,
and the field should stop asking for larger benchmarks in that category and
require non-benchmark evidence to carry the certification burden instead.

One boundary condition is worth stating, since it is where the result would
fail. Both theorems assume $\Nmax$ binds on approximately independent trials.
Adaptive or model-assisted elicitation could change the governing statistic. The
natural route to such a demonstration, however, tightens rather than loosens the
bound: if the evaluation is mediated by a single strategic policy the model
adopts once it infers it is under test, the chain from safety state to policy to
observations is Markov, and the data processing inequality caps the total
information any combination of sources carries by the capacity of that
bottleneck rather than by the sum of their individual ceilings. Models can
strategically underperform on evaluations they detect
\cite{vanderweij2025sandbagging}, so this is not hypothetical. We regard the
scaling of adaptive elicitation in the rare-harm regime as the most important
open question bearing on both theorems.

\subsection{Where the two regimes meet}

The two theorems share a boundary, and a second quantity locates a second
boundary of independent interest: the rate at which a clean sheet and a single
observed failure carry equal evidence.

Scored against the same hypothesis pair, one harmful output in $n$ trials gives
$\Lpos = (1/r)[(1-p_u)/(1-p_s)]^{n-1}$. At $n = 520$ and
$(p_u,p_s) = (0.001, 0.0005)$ this is $1.54$, against $\Lnull = 0.77$: the
single harm carries $0.625$ bits and the clean sheet $0.375$ bits, a ratio of
$1.7$. We state this plainly because an inflated version of the comparison is
easy to write. Setting $\Lnull$ beside a capability-model likelihood ratio of
$1/\varepsilon$ yields an apparent factor of eighteen, and that comparison is
invalid, since the two quantities answer different questions.

\begin{corollary}[Crossing rate]\label{cor:crossing}
For a benchmark of $n$ approximately independent trials at improvement ratio
$r$, the clean sheet and the single observed harm carry equal evidence at
\[
  p^{\times} \;\approx\; \frac{\ln(1/r)}{2n(1-r)} \;=\; \tfrac{1}{2}\,\pmin(r, n, r),
\]
with the clean sheet the more informative observation above $p^{\times}$ and the
single harm the more informative below it. At $n = 520$, $r = 0.5$ this gives
$p^{\times} = 1.33 \times 10^{-3}$; at $n = 10^{5}$ it gives
$6.93 \times 10^{-6}$.
\end{corollary}

Table~\ref{tab:limit} traces both quantities across four orders of magnitude.
The pattern has two halves and both are informative. Above the crossing rate the
clean sheet dominates, by a factor of $1.4$ at $p = 10^{-2}$ and $2.1$ at
$p = 5 \times 10^{-3}$: when harm is common enough to be expected, its absence
is genuinely surprising and therefore genuinely evidential. Below the crossing
rate the ordering reverses and then diverges, because $|\log_2 \Lnull| \to 0$
while $|\log_2 \Lpos| \to \log_2(1/r)$, which is exactly one bit at $r = 0.5$
regardless of $n$ or $p$.

\begin{table}[!htbp]
\centering
\begin{threeparttable}
\caption{\textbf{Two regimes, one hypothesis pair.} Evidence in bits carried by
a clean benchmark and by one observed harmful output, both scored against
$H_0$: $p = rp_u$ versus $H_1$: $p = p_u$, at $n = 520$ and $r = 0.5$. The
ordering reverses at $p^{\times} = 1.33 \times 10^{-3}$
(Corollary~\ref{cor:crossing}). Neither observation is universally the stronger.}
\label{tab:limit}
\small
\begin{tabular}{@{}lccccc@{}}
\toprule
$p_u$ & $10^{-2}$ & $5 \times 10^{-3}$ & $10^{-3}$ & $10^{-4}$ & $10^{-5}$ \\
\midrule
$|\log_2 \Lnull|$, clean sheet & 3.779 & 1.882 & 0.375 & 0.038 & 0.004 \\
$|\log_2 \Lpos|$, one observed harm & 2.772 & 0.879 & 0.625 & 0.963 & 0.996 \\
ratio, harm to clean sheet & 0.73 & 0.47 & 1.7 & 25.7 & 265.6 \\
\bottomrule
\end{tabular}
\end{threeparttable}
\end{table}

The crossing rate falls as $1/n$, so a larger benchmark not only gains power but
also extends the regime in which absence is worth citing further into the rare
tail. Further, one
observed harm is exactly zero evidence, $\Lpos = 1$, at $p = 2p^{\times}$, above
which a single harm in $n$ trials argues \emph{for} the safer hypothesis,
because fewer harms were seen than $H_1$ predicts.

Separately, and independent of rate, a reproduced harmful output addresses a
question no null result can. Write $H_1'$ for capable and $H_0'$ for incapable
of producing the output under any elicitation procedure. A reproduced finding
from a procedure with false-positive probability $\varepsilon$ gives
$\Lambda' = 1/\varepsilon$, while a null result gives $\Lambda' \approx 1$
whatever $n$ is, since under a rare-event distribution absence is nearly as
probable under $H_1'$ as under $H_0'$. 

\subsection{The same bound for any elicitation procedure}\label{sec:general}

The analysis so far has been written for a passive benchmark, but nothing in it
depends on the prompts being fixed in advance. Red-teaming practice extends well beyond benchmarks and
because it makes the bound considerably more favorable to skilled elicitation.

The step is to correct an attribution the notation has so far obscured. We have
written $p_u$ and $p_s$ as though a harm rate were a property of the model (it
is not). What an evaluation observes is conditional on the model, the elicitation
procedure, the scoring rule, the threat model, and the sampling protocol
together, and changing any of the last four changes the number without changing
the model at all. Write instead
\[
  q_1 = P(\text{procedure elicits harm} \mid H_1), \qquad
  q_0 = P(\text{procedure elicits harm} \mid H_0),
\]
for the two hypothesis-conditioned elicitation rates of a given procedure. A
null result across $n$ independent replications then gives
$\Lnull = [(1-q_1)/(1-q_0)]^{n}$ exactly, and the evidence it carries is
$n \kappa$ bits, where
\[
  \kappa = \bigl| \log_2 \bigl[ (1-q_1)/(1-q_0) \bigr] \bigr|
\]
is the per-trial discrimination of the procedure. The passive benchmark is the
special case $q_1 = p_u$, $q_0 = p_s = r p_u$, and the approximation
$\kappa \approx q_1(1-r)/\ln 2$ used earlier is a first-order expansion valid in
the rare-event regime at fixed $r$. Outside that regime the exact expression and
its approximation diverge, sometimes substantially, and the exact form should be
used. The threshold $\tau$ against which these quantities are judged is fixed
decision-theoretically in Section~\ref{sec:tau}.

Interestingly enough, two consequences follow which pull in opposite directions.

The first is encouraging; discrimination depends on the \emph{gap} between the
two rates rather than on the sample count, so a procedure that separates the
hypotheses sharply buys evidence very cheaply. A campaign with $q_1 = 0.9$ and
$q_0 = 0.1$ contributes $3.17$ bits from a single replication, which is what a
passive benchmark at $q_1 = 10^{-3}$ needs roughly 4,400 prompts to match. Two
such replications reach $\tau = 0.1$ and three reach $\tau = 0.01$. The
impossibility boundary shifts accordingly, since $\pmin$ scales inversely with
$\kappa$: raising discrimination, not enlarging the corpus, is the efficient
route across it. Theorem~\ref{thm:impossibility} is therefore an argument
against undiscriminating red teaming, as opposed to red teaming as a whole.

The second is a caution, and it is the reason attack success rate is the wrong
figure of merit \cite{chouldechova2025comparison}. A procedure with a high
elicitation rate under $H_1$ carries no evidence at all if its rate under $H_0$
is equally high. At $q_1 = q_0 = 0.9$, $\kappa$ is exactly zero: the procedure
succeeds against safe and unsafe models alike, so its failure to succeed on this
occasion says nothing about which it faced. At $q_1 = 0.9$, $q_0 = 0.85$, five
replications contribute $2.92$ bits, less than one replication of the
discriminating procedure above. Reporting $q_1$ alone, which is what an attack
success rate does, cannot distinguish these cases.

\begin{table}[!htbp]
\centering
\begin{threeparttable}
\caption{\textbf{Discrimination, not attack success, determines evidential
worth.} Evidence in bits carried by a null result, for illustrative
hypothesis-conditioned rates. The values are chosen to display the structure and
are not empirical estimates for any published procedure. Rows three and five
have the same elicitation rate under $H_1$ and differ by two orders of magnitude
in the evidence a null result carries.}
\label{tab:procedures}
\small
\begin{tabularx}{\textwidth}{@{}X l c c c c@{}}
\toprule
\textbf{Procedure} & \textbf{Trial unit} & $n$ & $q_1$ & $q_0$ & \textbf{bits} \\
\midrule
Passive benchmark, AdvBench-sized & prompt & 520 & $10^{-3}$ & $5\times10^{-4}$ & 0.38 \\
Passive benchmark, at $\Nmax$ & prompt & $10^{5}$ & $10^{-3}$ & $5\times10^{-4}$ & 72.19 \\
Adaptive campaign, discriminating & campaign & 1 & 0.90 & 0.10 & 3.17 \\
Adaptive campaign, discriminating & campaign & 5 & 0.90 & 0.10 & 15.85 \\
Adaptive campaign, non-discriminating & campaign & 5 & 0.90 & 0.85 & 2.92 \\
\bottomrule
\end{tabularx}
\end{threeparttable}
\end{table}

The trial unit carries the whole argument and must be stated explicitly, because
the three candidates are not interchangeable. For a passive benchmark the unit
is the prompt, and independence across prompts is the assumption the clustering
analysis of Section~\ref{sec:practice} interrogates. For an adaptive campaign
the unit is the \emph{campaign}, not the prompts inside it, so an expert who issues
a hundred probes while conditioning each on the last has run one trial, not a
hundred, since the probes are neither independent nor identically distributed.
Independent replication means running the procedure again from the start, with a
different operator or seed and no shared state. Counting adaptive probes as
independent trials inflates $n$ by whatever factor the campaign length happens
to be, and it is the single most likely way for this framework to be misapplied
in favor of the evaluator.

Estimating $q_1$ requires a positive control, meaning a reference model known to
possess the capability, against which the procedure's elicitation rate can be
measured directly. Laboratories routinely construct such references. Almost none
report the resulting rate, and without it a null result cannot be weighed at
all. If a model behaves differently when it infers it is under evaluation
\cite{vanderweij2025sandbagging}, then $q_1$ measured on the reference overstates
the procedure's discrimination against the actual target, and the null result is
worth less than the formula returns. That is the sharpest form of the
sandbagging concern for our framework.

\subsection{Where the threshold \texorpdfstring{$\tau$}{tau} comes from}\label{sec:tau}

The threshold is not arbitrarily chosen by hand. Let $L = C_{\mathrm{FN}}/C_{\mathrm{FP}}$
be the loss ratio, the cost of deploying a genuinely unsafe model relative to
that of an unnecessary delay. The Bayes-optimal rule deploys when
$P(H_1 \mid \text{data})\,C_{\mathrm{FN}} \le P(H_0 \mid \text{data})\,C_{\mathrm{FP}}$,
which rearranges to $O_1 \le 1/L$, and since $O_1 = O_0 \Lnull$,
\[
  \tau = \frac{1}{L\,O_0},
\]
fixed by the acceptable posterior risk and the prior odds, and it is this
quantity that Table~\ref{tab:ladder} holds fixed when converting a risk
tolerance into a sample size. The direction is the
one intuition demands: the more catastrophic a missed failure, the larger $L$,
the smaller $\tau$, and the more evidence a null result must carry.

The consequences cut both ways, which is the point of stating the rule
explicitly. A regulator treating an unsafe deployment as one hundred times
costlier than a delay, from even prior odds, needs $\tau = 10^{-2}$, demanding
$n \ge 915$ at $p = 10^{-2}$ but $n \ge 9{,}204$ at $p = 10^{-3}$ and
$n \ge 92{,}097$ at $p = 10^{-4}$. The first is comfortably within reach of
existing practice; the third is not. Fixing $\tau$ decision-theoretically
converts any stated risk tolerance directly into a sample-size floor, but
whether current benchmarks clear it depends entirely on the harm rate.

% ==========================================================================
\section{How far is current practice from the boundary?}\label{sec:practice}

Theorems~\ref{thm:adequacy} and~\ref{thm:impossibility} locate a boundary. This
section measures where eight widely used evaluation suites sit relative to it,
and finds them adequate on one side and far short on the other.

One thing to note is that the only benchmark-specific input to the calculations below
is the sample size $n$; harm rates are assumed, common across suites, and taken
from published estimates. These are therefore analytic properties of a design at
an assumed operating point and they are not a ranking of benchmark quality. The suites also
differ in estimand, and we treat that as part of the finding. XSTest \cite{rottger2024xstest} is an exaggerated-safety suite whose 250
safe prompts measure over-refusal rather than harm elicitation. SafetyBench
\cite{zhang2024safetybench} is a multiple-choice knowledge probe; the
$n = 2{,}100$ figure is its Chinese subset (\texttt{test\_zh\_subset.json}),
built by removing items with highly sensitive keywords and downsampling to 300
questions in each of seven categories to accommodate commercial API filtering
and rate limits. A single power axis is comparable across these suites only in
the weak sense that all are invoked to support the same kind of deployment
claim.

\subsection{Statistical power}

At the frontier model operating point ($p_{\text{unsafe}} = 0.01$, fifty percent
improvement), power ranges from 14.9 percent (XSTest, $n = 250$) to 59.3 percent
(SafetyBench, $n = 2{,}100$), with the two most widely cited adversarial
benchmarks, AdvBench ($n = 520$) \cite{zou2023universal} and HarmBench
($n = 400$) \cite{mazeika2024harmbench}, at 22.6 and 18.7 percent
(Figure~\ref{fig:power}). Achieving 80 percent power at this point requires
approximately 3,679 prompts per arm at one-sided $\alpha = 0.05$, or 4,670 per
arm two-sided.

It is important to consider that power and certification are different standards. Power concerns detecting a difference between two
models; certification concerns bounding one model's harm rate. A benchmark can
be underpowered for the first and adequate for the second, and at
$p = 10^{-2}$ several of these suites are exactly that. The comparison task is
the harder of the two, and reporting only the power figure understates what a
clean run establishes.

\begin{figure}[!htbp]
\centering
\includegraphics[width=\textwidth]{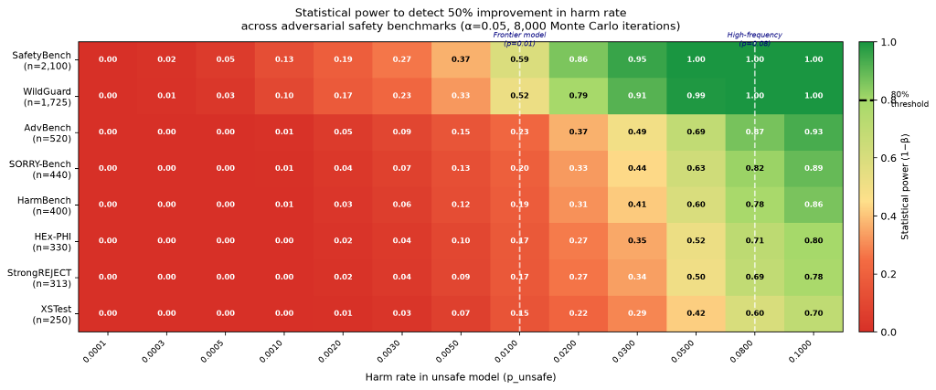}
\caption{\textbf{Statistical power across eight evaluation suites.} Power to
detect a fifty percent reduction in harm rate, from 8,000 Monte Carlo iterations
per operating point at $\alpha = 0.05$. Dashed lines mark the frontier operating
point ($p = 0.01$) and a high-frequency category ($p = 0.08$). At $p = 0.08$
half of the suites are adequately powered; at $p = 0.01$ none reaches 80 percent.}
\label{fig:power}
\end{figure}

\subsection{False certification}

When an evaluation observes zero harmful outputs, the true rate may still be
substantial. At $p = 0.001$, StrongREJECT ($n = 313$)
\cite{souly2024strongreject} yields a false-certification probability of 73.1
percent and AdvBench 59.4 percent. Reducing it below five percent at that rate
requires $n > 2{,}995$, and below one percent $n > 4{,}603$; the largest suite
reviewed here, at $n = 2{,}100$, reaches neither. At $p = 10^{-2}$, by contrast,
the same five percent standard needs only 299 prompts, which seven of the eight
suites clear.

What a clean sheet does establish is an upper bound, which is informative even when certification fails. The exact
Clopper-Pearson interval \cite{clopper1934use} is the formal instrument, and its
zero-event case is Hanley and Lippman-Hand's rule of three
\cite{hanley1983nothing}. Zero harms in 520 prompts gives a 95 percent
one-sided upper bound of 0.57 percent, roughly one in 174. Zero in 2,100 gives
0.14 percent, one in 701. Reporting such calibrated bounds costs a laboratory
nothing it does not already possess.

Table~\ref{tab:falsecert} shows the whole surface, since the boundary is a
statement about a region of it.

\begin{table}[!htbp]
\centering
\begin{threeparttable}
\caption{\textbf{False-certification probability across the operating surface.}
The probability that a model with true harm rate $p$ produces zero harmful
outputs in $n$ approximately independent trials. Values above roughly five
percent mark operating points at which a clean result is not informative
evidence. Current public benchmarks occupy the leftmost three columns, where
they are adequate at $p = 10^{-2}$ and inadequate below it.}
\label{tab:falsecert}
\small
\begin{tabular}{lcccccc}
\toprule
$P(k=0 \mid n, p)$, \% & $n = 250$ & $n = 520$ & $n = 2{,}100$ & $n = 5{,}000$ & $n = 10^{4}$ & $n = 10^{5}$ \\
\midrule
$p = 10^{-2}$ &  8.1 &  0.5 &  0.0 &  0.0 &  0.0 &  0.0 \\
$p = 10^{-3}$ & 77.9 & 59.4 & 12.2 &  0.7 &  0.0 &  0.0 \\
$p = 10^{-4}$ & 97.5 & 94.9 & 81.1 & 60.7 & 36.8 &  0.0 \\
$p = 10^{-5}$ & 99.8 & 99.5 & 97.9 & 95.1 & 90.5 & 36.8 \\
\bottomrule
\end{tabular}
\end{threeparttable}
\end{table}

\begin{figure}[!htbp]
\centering
\includegraphics[width=\textwidth]{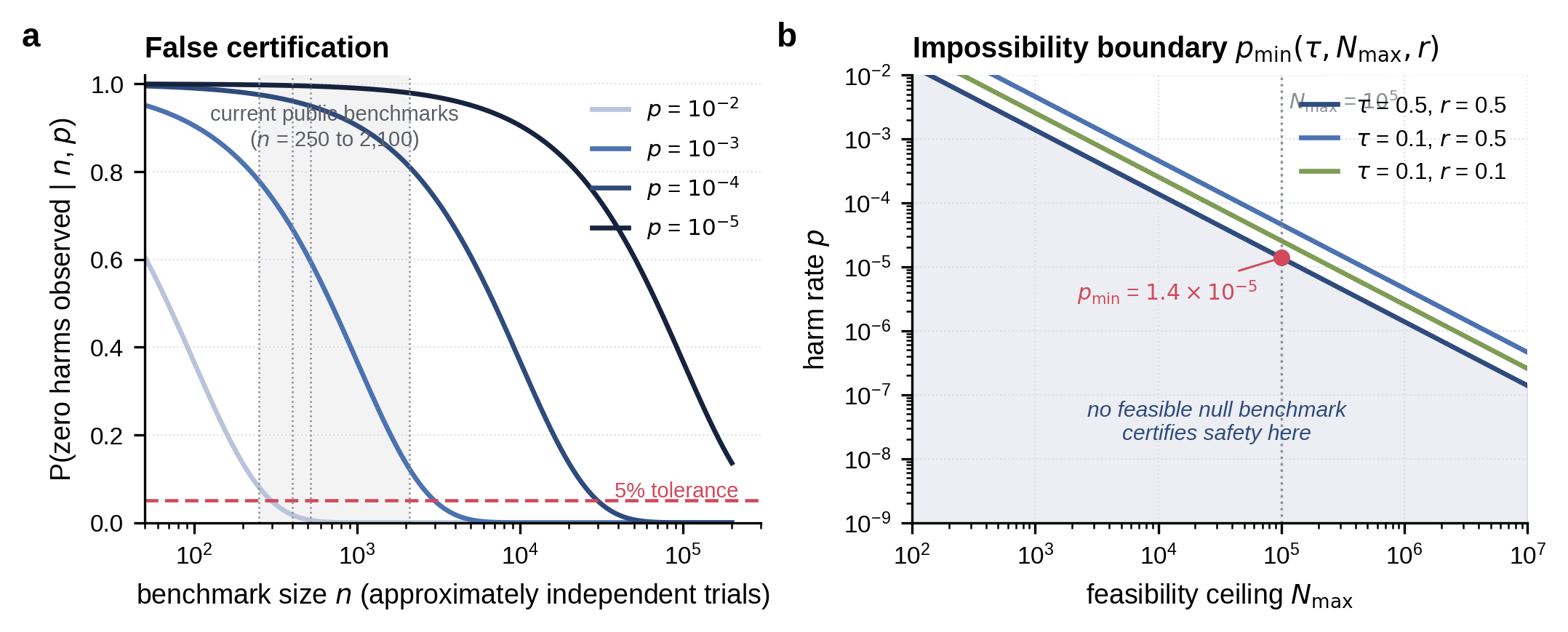}
\caption{\textbf{False certification and the boundary.} (\textbf{a})~False
certification against benchmark size for four harm rates, with the range of
current public benchmarks shaded. (\textbf{b})~The boundary
$\pmin(\tau, \Nmax, r)$ against the feasibility ceiling, for three combinations
of $\tau$ and $r$. Above a curve, Theorem~\ref{thm:adequacy} gives a finite
sufficient $n$; in the shaded region below, Theorem~\ref{thm:impossibility}
applies and no feasible benchmark certifies.}
\label{fig:boundary}
\end{figure}

\subsection{Clustering and effective sample size}

The calculations above treat the $n$ prompts as independent draws, but real
benchmarks are not built that way. Prompts are generated from a modest number of
templates and paraphrase families, so responses within a family are positively
correlated and the nominal count overstates the information carried.

The standard correction is the design effect,
$\mathrm{DEFF} = 1 + (m-1)\rho$, with $m$ the mean cluster size and $\rho$ the
intra-cluster correlation, giving effective sample size
$\neff = n/\mathrm{DEFF}$ \cite{kish1965survey}. Modest clustering bites hard:
$m = 10$ with $\rho = 0.1$ gives $\mathrm{DEFF} = 1.9$, and $m = 25$ with
$\rho = 0.1$ gives $3.4$, so $\neff \approx 0.29n$. The design effect is a
variance approximation, so substituting $\neff$ into a binomial likelihood is a
first-order heuristic rather than an identity; the principled route models the
per-prompt indicator as beta-binomial with intra-class correlation $\rho$
\cite{williams1975analysis}. Both routes agree on the direction, and the
direction is what matters here: clustering strictly raises $P(k=0)$ for any
$\rho > 0$, so it pushes the boundary of Theorem~\ref{thm:adequacy} upward and
enlarges the region where Theorem~\ref{thm:impossibility} applies.

A published disclosure illustrates why this is not merely a theoretical concern.
Anthropic's Claude 2 model card reports that among a held-out set of exactly 328
prompts, the model produced a response judged more harmful than a fixed refusal
reference in exactly four cases \cite{anthropic2023claude2}. Read at the prompt
level, $k = 4$ of $n = 328$, the prompts are the independent units and the exact
95 percent interval on the per-prompt rate is 0.33 to 3.09 percent, which is a
genuinely informative statement. But the same passage records that five
responses were sampled per prompt at $T = 1$, so scoring was performed over
1,640 responses while the numerator counts prompts. The per-response rate, which
is the quantity a deployment claim concerns since users receive responses rather
than prompts, is not recoverable: its numerator is not reported, and the card
notes that in one flagged case the model was disrupted in about half of its
sampled responses, which is intra-cluster correlation described in words. Two binomials exist here, they estimate different quantities, and the
disclosure supports the one further from the deployment claim. Reporting the
analysis unit and either $\neff$ or $\rho$ would resolve this at no cost, and
almost no disclosure currently does.

\subsection{Distributional validity}

Adequate power would not rescue an evaluation whose prompt distribution does not
represent deployment. Using TF-IDF vectorization and neural sentence embeddings
(all-MiniLM-L6-v2 \cite{reimers2019sbert}), we compared AdvBench, HarmBench, and
9,089 real user queries from LMSYS-Chat-1M \cite{zheng2024lmsys}. Benchmark
prompts are internally self-similar at 0.18 (TF-IDF) and 0.19 (neural), while
benchmark-to-deployment similarity is 0.057 and 0.055, a 3.1 to 3.5-fold excess
consistent across representations (Figure~\ref{fig:umap}). Sliced Wasserstein
distance and maximum mean discrepancy \cite{gretton2012mmd} confirm the
separation. Cross-benchmark distances (0.014 to 0.022) are comparable to or
exceed benchmark-to-deployment distances (0.014 to 0.020): the two primary
public benchmarks do not triangulate a deployment distribution, but sample from
a narrow, partially redundant region of adversarial prompt space.

\begin{figure}[htbp]
\centering
\includegraphics[width=0.72\textwidth]{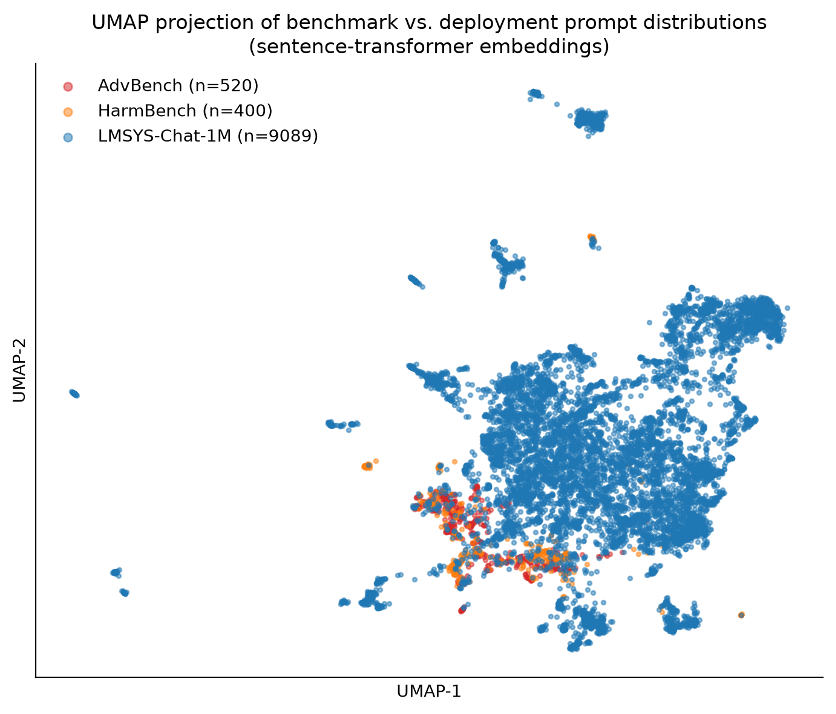}
\caption{\textbf{Benchmark and deployment prompt distributions.} Joint UMAP
projection \cite{mcinnes2018umap} of AdvBench, HarmBench, and a 9,089-query
sample of LMSYS-Chat-1M under sentence-transformer embeddings. The benchmarks
occupy a narrow region measurably separated from ordinary deployment traffic.
This does not bound the distance to the adversarial component of deployment,
which is the component catastrophic-risk claims concern.}
\label{fig:umap}
\end{figure}

The qualitative consequence needs no metric at all. Let $Q$ be the test
distribution and $P$ the deployment distribution, with $h$ indicating harm. A
zero-harm result on $n$ draws from $Q$ bounds $\mathbb{E}_Q[h]$ and, absent an
assumption linking $P$ to $Q$, places no upper bound on $\mathbb{E}_P[h]$
whatsoever: take $h$ supported on a region $A$ with $Q(A)$ arbitrarily small and
$P(A)$ arbitrarily close to one. A benchmark null result therefore says nothing
about deployment without a linking assumption, and the measured separation is
evidence that the assumption is not innocuous.

Worth noting here that LMSYS-Chat-1M samples ordinary use, which is one
component of a deployment mixture that also contains misuse-like requests,
deliberate attacks, and their multilingual, encoded, and role-play variants.
The measured separation establishes that these benchmarks do not represent
ordinary traffic, and nothing more. It does not bound the distance to the
adversarial component, and we make no such claim: an adversarial benchmark could
sit far from ordinary use and close to real attack traffic, which is what its
designers intend. Answering that question requires a corpus of real adversarial
traffic that is not, to our knowledge, publicly available at scale. We regard
constructing one as a higher priority for the field than enlarging the existing
benchmarks.

\subsection{Structural generalization}

A model that refuses a specific syntactic pattern has learned a surface
association, not the underlying harm concept, and adversaries circumvent such
associations trivially \cite{carlini2024aligned}. We term this the distinction
between \emph{behavioral safety}, refusal of specific formulations, and
\emph{structural safety}, robustness of harm avoidance as a latent property.
Published attacks using past-tense, multilingual, encoded, and role-play
variants \cite{liu2024autodan,andriushchenko2025pasttense,yong2024lowresource,
mehrotra2024tap}, together with the divergence between in-distribution and
out-of-distribution attack success once benchmarks saturate, indicate that the
two come apart \cite{broomfield2025structural,perez2022ignore}. Benchmark
performance is evidence of structural safety only to the degree that it
generalizes across semantically equivalent variants and novel attack families,
and the semantic consistency across variants together with the gap
between known and novel attack success rates should be reported alongside
canonical performance.

% ==========================================================================
\section{What an evaluation may claim}\label{sec:claims}

The framework's practical output is a mapping from what an evaluation did to
what it entitles a laboratory to say. 

\begin{table}[!htbp]
\centering
\begin{threeparttable}
\caption{\textbf{The claims ladder.} Required sample size, at $r = 0.5$ and
one-sided $\alpha = 0.05$, for each level of claim at four harm rates. Current
public benchmarks ($n \le 2{,}100$) support every claim in the table at
$p = 10^{-2}$, the weakest two at $p = 10^{-3}$, and only the upper bound below
that. The final row applies Theorem~\ref{thm:impossibility} at
$\Nmax = 10^{5}$.}
\label{tab:ladder}
\small
\begin{tabularx}{\textwidth}{@{}X cccc@{}}
\toprule
\textbf{Claim the evaluation licenses} & $p = 10^{-2}$ & $p = 10^{-3}$ &
$p = 10^{-4}$ & $p = 10^{-5}$ \\
\midrule
Harm rate bounded above at 95\% confidence & any $n$ & any $n$ & any $n$ & any $n$ \\
Belief shifts toward safety by $2\times$ ($\tau = 0.5$) & 138 & 1{,}386 & 13{,}862 & 138{,}629 \\
False certification held below 5\% & 299 & 2{,}995 & 29{,}956 & 299{,}572 \\
Belief shifts by $10\times$ ($\tau = 0.1$) & 458 & 4{,}602 & 46{,}049 & 460{,}514 \\
Belief shifts by $100\times$ ($\tau = 0.01$) & 915 & 9{,}204 & 92{,}097 & 921{,}028 \\
\midrule
Feasible at $\Nmax = 10^{5}$? & yes & yes & yes & no \\
\bottomrule
\end{tabularx}
\end{threeparttable}
\end{table}

The ladder is the paper's answer to the question of whether safety benchmarks
work. They work, at rates and sample sizes the table specifies, and they do not
work outside them. A laboratory running AdvBench on a category with a one
percent base rate may claim an order-of-magnitude update toward safety, which is
a substantial claim and stronger than anything current system cards assert. The
same laboratory running the same benchmark on a category at $10^{-4}$ may claim
only an upper bound of one in 174, and should say so rather than reporting that
no harmful outputs were observed.

\subsection{A reporting template}

The quantities an evaluation must publish before a null result can be weighed
are few, and none requires data a laboratory does not already hold at the moment
it evaluates. 

\begin{table}[!htbp]
\centering
\begin{threeparttable}
\caption{\textbf{Minimum reporting template for a red-team null result.} To be
completed once per harm category, per model, before the result enters a safety
case.}
\label{tab:template}
\small
\begin{tabularx}{\textwidth}{@{}l l X@{}}
\toprule
\textbf{Field} & \textbf{Symbol} & \textbf{Requirement} \\
\midrule
Harm category & --- & Named, with the threat model it operationalizes \\
Elicitation procedure & --- & Passive corpus, adaptive campaign, or other; stated \\
Trial unit & --- & Prompt, campaign, or replication; campaigns are not prompts \\
Nominal sample size & $n$ & Exact count of independent trials in the stated unit \\
Event count & $k$ & Exact integer, at the same analysis unit as $n$ \\
Generations per prompt & $m$ & Exact; state if greater than one \\
Intra-cluster correlation & $\rho$ & Estimated from the run, or an upper bound \\
Effective sample size & $\neff$ & $n / [1 + (m-1)\rho]$ \\
Upper confidence bound & --- & Exact Clopper-Pearson at 95\% on the nominal $(n, k)$ \\
Clustering-adjusted bound & --- & Beta-binomial or GEE; \emph{not} Clopper-Pearson on $\neff$ \\
Minimum detectable effect & $\delta_{\min}$ & Prespecified, not chosen post hoc \\
Observed power & $1 - \beta$ & Against $\delta_{\min}$, with sidedness stated \\
Evidentiary threshold & $\tau$ & With the loss ratio $L$ and prior odds $O_0$ that fix it \\
Elicitation rate under $H_1$ & $q_1$ & Measured against a positive control; state the control \\
Elicitation rate under $H_0$ & $q_0$ & Stated or bounded; $q_1$ alone is not sufficient \\
Per-trial discrimination & $\kappa$ & $|\log_2[(1-q_1)/(1-q_0)]|$, in bits \\
Achieved likelihood ratio & $\Lnull$ & $[(1-q_1)/(1-q_0)]^{n}$ exactly, reported as a number \\
Claim licensed & --- & The highest row of Table~\ref{tab:ladder} the design supports \\
Distributional coverage & --- & Distance to each deployment mixture component \\
Semantic consistency & $\sigma$ & Across prompt transformation types \\
Out-of-distribution gap & --- & Attack success on held-out attack families \\
Position relative to boundary & --- & Whether $p$ is suspected below $\pmin(\tau, \Nmax, r)$ \\
\bottomrule
\end{tabularx}
\end{threeparttable}
\end{table}

Two rows deserve comment here. The Clopper-Pearson interval is defined for an integer
number of independent Bernoulli trials, so it must be computed on the nominal
$(n,k)$; $\neff$ is generally fractional and is not a sample size in the sense
the interval requires. Where clustering is material, report the exact interval
and a beta-binomial or generalized-estimating-equations interval
that models the dependence directly. The final row has no counterpart in current
practice: for a category suspected to lie below $\pmin$, the honest entry is not
a larger benchmark but a declaration that null results cannot certify this
category at feasible scale, together with the non-benchmark evidence being
relied on instead.

Filling the template is largely a matter of reporting rather than of new
measurement, and a preliminary review of public frontier disclosures suggests
the gap is wide. Across nine such documents \cite{openai2023gpt4,
openai2024gpt4o, anthropic2023claude2, anthropic2024claude3, anthropic2025claude4,
gemini2023, gemini15, touvron2023llama2, grattafiori2024llama3} we found that
only one endpoint reports both an exact numerator and an exact denominator for a
binary harm rate, and none reports the dependence structure needed to interpret
it. We develop
that audit, and the transparency instrument behind it, separately.

\subsection{Four reforms}

Pre-registration of the harm categories, target detectable rates, sampling
strategy, and decision criteria, before deployment decisions are made, prevents
post-hoc reframing and establishes a public record. Calibrated reporting means
publishing Table~\ref{tab:template} per category and stating the licensed claim
rather than the raw observation. Held-out corpora, maintained by independent
bodies under controlled access, address benchmark saturation, which is a
structural consequence of public evaluation sets. Independent evaluation
addresses the conflict of interest that no methodological reform resolves, and
its distinctive role would be to assess evaluation design validity, meaning
whether power, coverage, and generalization testing suffice for the claim being
made.

Several institutions have begun building this infrastructure, including METR's
pre-deployment capability evaluations, the UK AI Security Institute's model
evaluations programme, and Anthropic's Responsible Scaling Policy commitments
\cite{anthropic2024rsp}. Mapping these against the three conditions reveals a
consistent pattern: they address structural generalization more directly than
statistical sufficiency or distributional validity. That is not a criticism of
intent. It reflects the absence of field-level standards of the kind this
framework supplies.

% ==========================================================================
\section{Limitations}\label{sec:limitations}

We state limitations in one place rather than distributing them, because their
aggregate is the honest measure of what this paper establishes.

\paragraph{The mathematics is elementary.} The zero-numerator bound underlying
both theorems dates to 1983 \cite{hanley1983nothing}. We claim no mathematical
novelty for this. Our contribution is the consequence drawn for evaluation design and
the identification of the boundary in both directions.

\paragraph{The power figures dependence on assumed operating points:} Published
harm-rate estimates for frontier models are sparse, and our anchors are
defensible rather than authoritative. The qualitative conclusions, adequacy at
$10^{-2}$ and inadequacy at $10^{-4}$, are robust across the plausible range;
the specific percentages are not.

\paragraph{Suites being heterogeneous in estimand:} XSTest measures
over-refusal and SafetyBench is multiple-choice. Placing eight suites on one
power axis is a deliberate simplification, defensible only because all eight are
invoked for the same kind of deployment claim.

\paragraph{The distributional gap covers one mixture component:} As stated previously, LMSYS-Chat-1M
samples ordinary use, so the measured separation establishes only that these
benchmarks do not represent ordinary traffic. It does not bound the distance to
the adversarial component.

\paragraph{The generalization depends on two estimated inputs.}
Section~\ref{sec:general} replaces a sample count with a pair of
hypothesis-conditioned elicitation rates, and both must be estimated rather than
assumed. The rate under $H_1$ requires a positive control whose validity as a
proxy is itself an assumption; the rate under $H_0$ is harder still, since it
asks how often the procedure fires against a model that is in fact acceptable.
Where $q_0$ cannot be estimated, an upper bound on it yields a lower bound on
discrimination, which is the conservative direction. The illustrative values in
Table~\ref{tab:procedures} display the structure of the result and are not
empirical estimates of any published procedure.

\paragraph{Both theorems have a stated falsifier.} Their bite depends on
$\Nmax$ binding on approximately independent trials. A demonstration that
adaptive elicitation achieves better than $\mathcal{O}(1/p)$ scaling in the
rare-harm regime would move the boundary. The bottleneck argument of
Section~\ref{sec:ceiling} suggests the natural route tightens rather than
loosens it, but that argument is not a proof of tightness by itself.

% ==========================================================================
\section{Discussion}

We defined the ceiling abstractly and
instantiated it only for the passive benchmark null result. The same construct
should be computable for other evidence types safety cases cite without
comparable statistical treatment: interpretability findings, human uplift
trials, whose statistical structure likely changes the closed form rather than
merely its parameters, and structured expert elicitation. Each requires its own
model of how the evidence is generated. This would be the natural next step.

Building on an adjacent line of work, if each cited evidence source carried
a computable ceiling, a safety case asserting posterior confidence beyond the
aggregate ceiling of its evidence would be unsound in a precise sense. We are
not confident this survives contact with how real safety cases combine
qualitative and quantitative arguments, and we have not formalized how ceilings
from heterogeneous types should aggregate. We state it as an open research subject.

% ==========================================================================
\section{Conclusions}

Adversarial evaluation has identified genuine vulnerabilities
\cite{ganguli2022redteaming, perez2022redteaming}, driven genuine safety
improvements, and provided the only systematic check on the deployment of
manifestly dangerous systems. Safety benchmarks are not worthless, and the
emerging view that they are is an overcorrection. What they are is instruments with a computable resolution.
Above a calculable harm rate, a benchmark of modest and affordable size shifts
belief toward safety by a factor a laboratory can state in advance, and a clean
result is then the stronger of the two possible observations. Below that rate,
no benchmark of feasible size does so, and the field should stop trying and
carry the certification burden with other evidence.

The boundary between the two is not a matter of judgment or of methodological
taste. It follows from one closed-form expression, it can be computed before an
evaluation is run, and it tells a laboratory what its evaluation will be able to
claim while there is still time to change the design, subject to the
qualifications set out in Section~\ref{sec:limitations}. That is the discipline
this paper proposes: not that safety evaluations be abandoned, nor that they be
trusted as they stand, but that every result be reported alongside the claim it
supports.

% ==========================================================================
\section*{Methods}
\addcontentsline{toc}{section}{Methods}

\paragraph{Monte Carlo power simulation.} Power was estimated via 8,000
parametric Monte Carlo iterations per $(n, p_{\text{unsafe}}, p_{\text{safe}})$
triple, drawing harm counts from $\mathrm{Binomial}(n,p)$ and applying a
one-sided two-sample proportions $z$-test at $\alpha = 0.05$. Rates were
anchored to published estimates: high-frequency ($p = 0.08$
\cite{ganguli2022redteaming}), mid-frequency ($p = 0.025$, HarmBench frontier
estimates \cite{mazeika2024harmbench}), low-frequency ($p = 0.003$), and the
frontier scenario ($p_{\text{unsafe}} = 0.01$).

\paragraph{Analytic quantities.}
The false-certification probability is $P(k=0 \mid n,p) = (1-p)^{n}$, and the
sample size holding it below a tolerance $\phi$ is
\[
  n \;=\; \left\lceil \frac{\ln \phi}{\ln(1-p)} \right\rceil .
\]
Sample sizes for a likelihood-ratio threshold $\tau$ follow~\eqref{eq:nreq}
rather than the expression above,
\[
  n \;=\; \left\lceil
    \frac{\ln \tau}{\ln\!\bigl[(1-p_{u})/(1-p_{s})\bigr]}
  \right\rceil ,
  \qquad p_{s} = r p_{u},
\]
the two criteria differing by a factor of $(1-r)^{-1}$ to first order, which is
$2$ at $r = 0.5$. Clopper--Pearson bounds were computed from the Beta
distribution; for zero events the one-sided $95$ percent upper bound reduces to
$1 - \alpha^{1/n}$ and the two-sided to $1 - (\alpha/2)^{1/n}$. Both are
reported where they differ materially. 

\paragraph{Distributional analysis.} AdvBench ($n = 520$), HarmBench
($n = 400$), and 9,089 deduplicated first-turn user queries sampled at random
from LMSYS-Chat-1M were vectorized two ways: TF-IDF (bigrams, 20,000 features)
reduced to 100 dimensions via truncated SVD, and neural sentence embeddings
(all-MiniLM-L6-v2, 384 dimensions). Sliced Wasserstein distance used 500 random
projections; MMD used an RBF kernel with median heuristic bandwidth. The joint
UMAP projection used cosine metric, \texttt{n\_neighbors} = 15,
\texttt{min\_dist} = 0.1.

\paragraph{Software.} Python 3.12 (numpy 2.4.4, scipy 1.17.1, scikit-learn
1.8.0, statsmodels 0.14.6). Code at
\url{https://github.com/hackwither/ai-redteam-evidential-limits}. All datasets
used are publicly available.

% ==========================================================================
\section*{Data availability}
\addcontentsline{toc}{section}{Data availability}

AdvBench: \url{https://github.com/llm-attacks/llm-attacks}.
HarmBench: \url{https://github.com/centerforaisafety/HarmBench}.
LMSYS-Chat-1M (gated):
\url{https://huggingface.co/datasets/lmsys/lmsys-chat-1m}.
Code: \url{https://github.com/hackwither/ai-redteam-evidential-limits}.
No proprietary data was used.

% ==========================================================================
\bibliographystyle{unsrt}
\bibliography{refs}

\end{document}